\documentclass[10pt,twocolumn,letterpaper]{article}

\usepackage{cvpr}
\usepackage{times}
\usepackage{epsfig}
\usepackage{graphicx}
\usepackage{amsmath}
\usepackage{amssymb}
\usepackage{amsthm}
\usepackage[normalem]{ulem}
\usepackage{mathtools} 

\usepackage[pagebackref=true,breaklinks=true,letterpaper=true,colorlinks,bookmarks=false]{hyperref}

\cvprfinalcopy 

\def\cvprPaperID{500} 

\newtheorem{theorem}{Theorem}

 \newtheorem{lemma}{Lemma}

\ifcvprfinal\fi
\begin{document}

\title{The Multiverse Loss for Robust Transfer Learning}
\author{Etai Littwin \ and Lior Wolf\\ The Blavatnik School of Computer Science\\Tel Aviv University}
\maketitle

\begin{abstract}
Deep learning techniques are renowned for supporting effective transfer learning. However, as we demonstrate, the transferred representations support only a few modes of separation and much of its dimensionality is unutilized. In this work, we suggest to learn, in the source domain, multiple orthogonal classifiers. We prove that this leads to a reduced rank representation, which, however, supports more  discriminative directions. Interestingly, the softmax probabilities produced by the multiple classifiers are likely to be identical. Experimental results, on CIFAR-100 and LFW, further demonstrate the effectiveness of our method. 
\end{abstract}

\section{Introduction}

One of the hallmarks of the recent success of deep learning methods in computer vision is the ability to learn effective representations in one domain and apply these on another domain
~\cite{Krizhevsky12, Girshick14, FeiFei06, Yang07,Orabona09, Tommasi10, Kuzborskij13a}. The source and target domains might differ in the underlying probability distribution, the imaging modality, and, often, in the task performed. A striking example is image captioning~\cite{Karpathy}, in which image representations trained on ImageNet~\cite{imagenet} are transferred along with word embeddings trained on Wikipedia and other corpora~\cite{Word2Vec} in order to solve a seemingly complex task of describing images with sentences.

Another task where transfer learning has been shown to be effective is face recognition. In this task, face representations are trained on large datasets collected from social networks or search engines. The representations are trained to solve the multiclass classification problem using a cross entropy loss and are then transferred to a different domain, e.g., the celebrity images of the LFW dataset~\cite{lfw}. Moreover, the task changes post-transfer to face verification (same/not-same). 

An effective algorithm for face verification based on engineered or learned representations is the Joint Bayesian (JB) method~\cite{jb}. JB, similarly to other Bayesian methods, such as Linear Discriminant Analysis (LDA), is based on the interplay between the within class covariance matrix $S_w$ and the between class covariance matrix $S_b$. We prove (Thm.~\ref{the:jb}) that JB fails to be discriminative whenever LDA fails, i.e., when the Fisher ratio (Eq.~\ref{eq:fr} below) is low.

When empirically observing the spectrum of Fisher ratios associated with the transferred representations, we noticed that only a handful of the generalized eigenvectors of $S_b$ and $S_w$ present large eigenvalues. The other directions are therefore non-discriminative and the representation can be considered flat.

To amend this situation, we propose to employ, in the source domain, a generalization of the cross entropy loss. In this generalization, multiple sets of classifiers are learned, such that the group of classifiers for each class is orthogonal. Each set of classifiers is trained using a separate cross entropy loss, and gives rise to its own set of probabilities.

When performing such training a few non-trivial properties emerge: (i) For each training sample, the vector of probabilities obtained is identical across the classifier sets; (ii) The dimensionality of the representation is reduced and (iii) The Fisher Spectrum displays multiple directions with high Fisher scores. In a series of theorems, we expose how the new loss leads to these properties. 

Finally, we demonstrate experimentally the both the effectiveness of our method and the consequences of the emerging properties. For example, using a single network, we obtain 2nd best results for a single network on LFW~\cite{lfw}. This is achieved using a a training set that is a few orders of magnitude smaller than that of the leading literature network~\cite{facenet}, and using a very compact representation of only 51 dimensions.

\section{Related work}

Compound losses for training deep neural networks that are created by combining multiple losses are now commonplace. In the very deep GoogLeNet network~\cite{GoogleLeNet} multiple cross entropy losses are distributed at different intermediate layers of the deep network in order to help avoid vanishing gradients. In contrast, our work supports multiple cross entropy losses at the top layer and for different reasons.

In many other cases, multiple losses are used in order to support multiple tasks by the same network. For example, in object detection and object segmentation, the location of the object is recovered jointly with the associated detection probability~\cite{fcn,fastrcnn,multibox}. This is in contrast to our case, where the same loss is used multiple times in order to improve the performance of one task.

In our work, we create multiple losses by constructing  multiple top classification layers on top of a shared network representation. Each classification layer has one output neuron per class. The weights from the representation to this neuron are the classifier weights for this specific classifier. In order to enforce multiplicity among the classifiers of the same class, we add an orthogonality constraint, which is enforced either in the representation space or in a Fisher spectrum aligned space. A number of ways to encourage diversity in a classifier ensemble by enforcing orthogonality have been studied in the machine learning literature~\cite{brown2005diversity} and in computer vision~\cite{levy2012minimal}. However, note that in our case orthogonality does not lead to diversity since all classifiers end up presenting the same set of probabilities.

A prominent example of the success of transfer learning can be seen in the task of face recognition. Starting with the work of Taigman et al.~\cite{deepface}, a neural network has been employed for extracting representations from face images that are shown to outperform humans. Sun et al.~\cite{deepid1,deepid2,deepid2plus,deepid3} further improve the state-of-art by extracting features from multiple face patches, incorporating architectures into the domain of deep face recognition that are inspired by recent architectures that are used for object recognition~\cite{vgg}, and most relevant to our work, combining, during training, both classification and contrastive loss. Another recent work~\cite{facenet} further improves the training criterion by using a triplet cost to increase the discriminability between identities. The idea presented here, of combining multiple copies of the same loss, was not pursued in previous works.

The deep face networks mentioned above, are all trained on large scale proprietary datasets, which are not publicly available. Yi et al.~\cite{casia} built a publicly available dataset by mining images from the internet. Furthermore, they demonstrated the quality of the data collected by training a state-of-the-art network on it. Their network architecture is similar to that of the VGG model~\cite{vgg}. JB is used to effectively enhance performance. In our work, we use the same architecture suggested in~\cite{casia} as the basis of our face recognition experiments. We also employ JB to learn similarities for faces and other objects. A recent paper using JB outside the domain of face recognition is~\cite{cardb}.

\section{Preliminaries and notation}

\begin{table}[t]
\begin{center}
\begin{tabular}{|c|p{2.6in}|}
\hline
Symbol & \\
\hline
$c$ & Number of classes, typically indexed by $j$.\\
$n$ & Number of data points, typically indexed by $i$.\\
$y$ & $n \times 1$ vector of labels. Each label is $y_i$.\\
$d$ & Dimensionality of the representation vectors.\\
$D$ & $d\times n$ features matrix.\\ 
$d_i$ & A column of $D$, the representation of sample $i$.\\
$F$ & $d\times c$ classifier matrix of weights.\\
$f_i$ & A column of $F$; A normal to the separating hyperplane of class $i$.\\
$b$ & $c\times 1$ vector of parameters (biases).\\
$L_i$ & Loss function value evaluated for data point $i$.\\
$L$ & The aggregated loss function: $\sum_{i=1}^{n}L_i$.\\
$F^*,b^*$& (any) global minimizers of $L(D,y)$.\\
$L^*(D,y)$& The minimum value of $L$ given $D,y$,  i.e L($F^*,b^*,D,y$).\\
$K$ & The linear kernel matrix of the data: $DD^{\top}$.\\
$\mathbf 1_c $ & An all 1 vector of length $c$: $[1,1 \hdots 1]_{c \times 1}$.\\
$p_i(j)$ & Vector of probabilities associated with $d_i$.\\
$S_b$ ($S_w$) & The between (within) class covariance matrix.\\

\hline
\end{tabular}
\end{center}
\caption{Summary of notations.}
\label{tab:terms}
\end{table}

The notations used in this work are summarized in Tab.~\ref{tab:terms}. $n$ training samples, indexed by $i=1\hdots n$ are represented, using a network of any depth as ``neural code'' vectors of length $d$, $D_{d\times n}=[d_1\hdots d_n]$. Each sample is associated with a label $y_i \in [1\hdots c]$. 

Classification is performed by projecting the representations $d_i$ by a $d\times c$ classifier matrix $F=[f_1 \hdots f_c]$ and adding biases $b \in \mathbb{R}^c$. Softmax probabilities are obtained as $p_i(y_i) = \frac{e^{d_{i}^\top f_{y_i}+b_{y_i}}}{\sum_{j=1}^{k}e^{d_{i}^\top f_{j}+b_{j}}}$.

The training loss of a single example is the negative-log likelihood and is a function of the classifier parameters $F$,$b$, the representation $D$, and the labels $y$: $L_i(F,b,D,y) = -\log p_i(y_i)$. The aggregated cross entropy loss is $L(F,b,D,y) = \sum_{i=1}^n L_i(F,b,d_i,y_i)$.

The loss function $L$ is a convex function of $F,b$~\cite{Cover:2006:EIT:1146355}. $F$ and $b$ do not define the mapping from sampled $d_i$ to probability vectors $p_i$ in a unique way, and there are multiple minimizers for $L$ as the following lemma shows.

\begin{lemma}
The minimizers $F^*,b^*$ of $L$ are not unique, and it holds that for any vector $v \in \mathbb{R}^c$ and scalar $s$, the solutions $F^*+v\mathbf 1_c^\top$, $b^*+s\mathbf 1_c$ are also minimizers of $L$.
\label{lem:uniqness}
\end{lemma}
\begin{proof}
denoting $V = v\mathbf 1_c^\top$ , $\mathbf{s}=s\mathbf 1_c$, 
\begin{multline}\label{eq:2}
L({F^*}+V,{b^*}+\mathbf{s},D,y) =\\
-\sum_{i=1}^{n}log(\frac{e^{d_{i}^\top f_{y_i}+d_{i}^\top v+b_{y_i}+s}}{\sum_{j=1}^{c}e^{d_{i}^\top f_{j}+d_{i}^\top v+b_{j}+s}}) \\
=-\sum_{i=1}^{n}log(\frac{e^{d_{i}^\top v+s}e^{d_{i}^\top f_{y_i}+b_{y_i}}}{\sum_{j=1}^{c}e^{d_{i}^\top v+s}e^{d_{i}^\top f_{j}+b_{j}}})\\
=-\sum_{i=1}^{n}log(\frac{e^{d_{i}^\top v+s}e^{d_{i}^\top f_{y_i}+b_{y_i}}}{e^{d_{i}^\top v+s}\sum_{j=1}^{c}e^{d_{i}^\top f_{j}+b_{j}}})\\
=-\sum_{i=1}^{n}log(\frac{e^{d_{i}^\top f_{y_i}+b_{y_i}}}{\sum_{j=1}^{c}e^{d_{i}^\top f_{j}+b_{j}}})=L({F^*},{b^*},D,y)
\end{multline}
\end{proof}

In this work, we study the compound loss that is obtained as $\sum_{r=1}^m L(F^r,b^r,D,y)$ for $m$ different sets of classifiers $F^r,b^r$. More specifically, let the set of classifier parameters be ${F}^1=\left[f_1^1 \hdots f_c^1\right]$,${b}^1$, ...${F}^m=\left[f_1^m \hdots f_c^m\right]$,${b}^m$, we enforce orthogonality for each class. This is done either in the conventional way: $\forall jrs~~f_j^{r\top} f_j^s = 0$, or in the domain of the within class covariance matrix $\forall jrs~~f_j^{r\top} S_w f_j^s = 0$. We call the second type of orthogonality ``$S_w$-orthogonality''.

The $S_w$ orthogonality is directly related to our goal of improving the number of distinct discriminative directions, as captured by the Fisher ratios. This is explored in Sec.~\ref{sec:fisher}. It resembles, other methods that down-regulate the contribution of the directions in the vector space that account for much of the within class
covariance, such as WCCN~\cite{wccn}.

In practice, this orthogonality is enforced by adding loss terms of the form $\lambda |f_j^{r\top} f_j^s|$ or $\lambda |f_j^{r\top} S_w  f_j^s|$. The value of $\lambda$ used throughout our experiments is 0.005, which is, for comparison, 10 times larger than the weight decay used during training. This value is high enough to ensure solutions that are very close to orthogonality (normalized dot products lower than $10^{-3}$) in all of our experiments. Higher weights might hinder an effective exploration of the parameter space during optimization.

For the $S_w$ orthogonality, $S_w$ depends on the representation and is estimated for each train mini-batch separately. In all experiments, a mini-batch of 200 samples was used. While the values of $S_w$ change between mini-batches, we found the estimations to be reliable.

Since multiple copies of the same loss are used, we term our loss ``the multiverse loss''. The choice of term is further motivated by the property, discussed below, that all copies are different (due to orthogonality) but provide the same probabilistic outcome.

\section{Properties of the learned representation}

When employing the multiverse loss $\sum_{r=1}^m L(F^r,b^r,D,y)$ for training the neural network, under either orthogonality constraint, the learned representation displays a few desirable properties. The first property is that for every two classifiers $F^r$,$b^r$ and $F^s$,$b^s$ the parameters are intimately related. The nature of this link depends on the rank of $D$. For a full rank $D$, the solutions are highly constrained, which can be seen as a very restrictive form of regularization. This leads to a lower rank representation, where orthogonal solutions are linked by rank-1 modifications.

We will be using the following Lemma in order to prove Thm~\ref{thm:singleuptorankone}.

\begin{lemma} 
Let $K = \sum_{i=1}^n d_i d_i^\top$ be a full rank $d \times d$ matrix, i.e., it is PD and not just PSD, then for all vector $q \in \mathbb{R}^n$ such that $\forall i~~~q_i>0$, the matrix $\hat K = \sum_{i=1}^n q_i d_i d_i^\top$ is also full rank. 
\end{lemma}
\begin{proof}
For every vector $v\in \mathbb{R}^d$,  
$v^\top \hat K v \geq (min_i q_i)v^\top K v > 0$.
\end{proof}

The following theorem links any two optimal solutions in the case in which $D$ is full rank. Note that the orthogonality constraint is not assumed.

\begin{theorem}
Assume the minimal loss $L^*(D,y)$ is obtained at two solutions $F^1,b^1$ and $F^2,b^2$. If $rank(D)=d$, then there exists some vector $v \in \mathbb{R}^c$ and some scalar $s$ such that $F^1 - F^2 = v\mathbf 1_c^\top$ and $b^1 - b^2 = s\mathbf 1_c$. 
\label{thm:singleuptorankone}
\end{theorem}
\begin{proof}
For simplicity we prove the case where $b^1=b^2=\mathbf{0}$, the case where $b^1,b^2\neq \mathbf{0}$ is similar. Let $\Psi = [\psi_1,\psi_2,\hdots,\psi_c]=F^2-F^1$, and let $\psi$ denote the concatenation of the column vectors $\psi_j$ into a single column vector. Given that $F^1,F^2$ achieve minimal loss, from convexity it must hold that:
\begin{equation}
\psi^T \nabla^2 L(D,y)\bigg|_{F^1}\psi = \psi^T \frac{\partial L(D,y)^2}{\partial F \partial F}\bigg|_{F^1}\psi =0
\end{equation}
where $\nabla^2 L^*(D,y)$ is the hessian of the loss. 
We will show that in order for $\psi$ to lie in its kernel it must hold that $\psi_1=\psi_2\hdots =\psi_c$. Recall that $p_{i}(j) $
is the vector of softmax probabilities associated with $d_i$. 
\begin{equation} \label{eq:14}
\frac{\partial}{\partial F_{ju}}L(D,y) = -\sum_{i=1}^{n}d_{iu}p_{i}(j)-\sum_{i,y_{i}=u}d_{iu}
\end{equation}
\begin{multline}
\frac{\partial^{2}}{\partial F_{ju}F_{j'v}}L(D,y) =\\ -\sum_{i=1}^{n}d_{iu}d_{iv}p_{i}(j)(\delta_{j=j'}(1-p_{i}(j))-\delta_{j\neq j'}p_{i}(j'))\\
\end{multline}

Therefore:
\begin{multline} \label{eq:hess1}
\psi^{\top}\frac{\partial^{2}}{\partial F\partial F}L(D,y)\psi =
\sum_{j=1}^{c}\psi_{j}^{\top}\sum_{i=1}^{n}d_{i}d_{i}^{\top}p_{i}(j)(1-p_{j}(u))\psi_{j}\\
- \sum_{j=1}^{c}\sum_{j'\neq j}\psi_{j}^{\top}\sum_{i=1}^{n} d_{i}d_{i}^{\top}p_{i}(j)p_{j'}(v)\psi_{j'}
\end{multline}
Since $(1-p_{i}(j))=\sum_{j'\neq j}p_{i}(j')$, the first term of Eq.~\ref{eq:hess1} can be written as follows:
\begin{multline}
\sum_{j=1}^{c}\psi_{j}^{\top}\sum_{i=1}^{n}d_{i}d_{i}^{\top}p_{i}(j)(1-p_{i}(j))\psi_{j}\\
=\sum_{j=1}^{c}\sum_{j'\neq j}\psi_{j}^{\top}\sum_{i=1}^{n}d_{i}d_{i}^{\top}p_{i}(j)p_{i}(j')\psi_{j}\\
=\sum_{j=1}^{c}\sum_{j'=j+1}^{c}[\psi_{j}^{\top}\sum_{i=1}^{n}d_{i}d_{i}^{\top}p_{i}(j)p_{i}(j')\psi_{j}\\
+ \psi_{j'}^{\top}\sum_{i=1}^{n}d_{i}d_{i}^{\top}p_{i}(j)p_{i}(j')\psi_{j'}]
\end{multline}
Similar manipulation can be done with the second term of Eq.~\ref{eq:hess1}:
\begin{multline}
- \sum_{j=1}^{c}\sum_{j'\neq j}\psi_{j}^{\top}\sum_{i=1}^{n} d_{i}d_{i}^{\top}p_{i}(j)p_{i}(j')\psi_{j'}=\\
-\sum_{j=1}^{c}\sum_{j'=j+1}^{c}2\psi_{j}^{\top}\sum_{i=1}^{n} d_{i}d_{i}^{\top}p_{i}(j)p_{i}(j')\psi_{j'}
\end{multline}
Adding the two term we get:
\begin{multline} \label{eq:hessian}
\psi^{\top}\frac{\partial^{2}}{\partial F\partial F}L(D,y)\bigg|_{F^1}\psi =\\
\sum_{j=1}^{c}\sum_{j'=j+1}^{c}(\psi_{j}-\psi_{j'})^{\top}\sum_{i=1}^{n} d_{i}d_{i}^{\top}p_{i}(j)p_{i}(j')(\psi_{j}-\psi_{j'})
\end{multline}
Since $\forall i,j~~p_{i}(j)>0$ and since $rank(D)$ is full,  $\sum_{i=1}^{n} d_{i}d_{i}^{\top}p_{i}(j)p_{i}(j')$ is PD. Eq ~\ref{eq:hessian} is therefor the sum of positive values, and can only vanish if and only if $\psi_j=\psi_{j'}$ for all $j,j'$. 
\end{proof}

In our method, we require that the multiple solutions found $F^1$,$F^2$ (possibly more) lead to orthogonal (or $S_w$-orthogonal) separating hyperplanes for each class. The theorem below shows that unless $D$ is degenerate, this requirement leads to either an increase of the total loss, or to a very specific and limiting type of regularization on $F^1$. Such a stringent regularization would hinder effective learning. For convenience, we state and prove Thm.~\ref{thm:Fv},~\ref{thm:bound1},~\ref{thm:multiverse} for the case of conventional orthogonality. The analog theorems for $S_w$-orthogonality are stated in the same way, and proven similarly, after applying the transformation $S_w^\frac{1}{2}$.

\begin{theorem}
Assume that $rank(D)=d$, that $d<c$, and that the minimal loss $L^*(D,y)$ is obtained at a solution $F^1,b^1$. If there exists a second minimizer $F^2,b^2$ such that for all $j\in [1...c]$ the orthogonality constraint $f_j^1 \perp f_j^2$ holds, then $F^1$ admits to a stringent second order constraint.
\label{thm:Fv}
\end{theorem}
\begin{proof}
Since $rank(D)=d$, Thm.~\ref{thm:singleuptorankone} implies that there exists some vector  $v \in \mathbb{R}^c$ and scalar $s$ such that $\forall j,f^{2}_{j}=f^{1}_{j}+v$. The orthogonality constraint demands that $\forall j, (f^{1}_{j}+v)^{\top}f^{1}_{j}=0$. In matrix form:
\begin{equation} \label{eq:Fv}
F^{1\top}v  = -
\begin{pmatrix}
||f^{1}_{1}||^2\\
||f^{1}_{2}||^2\\
\vdots  \\
||f^{1}_{c}||^2\\
\end{pmatrix}
\end{equation}
This set of $c$ linear equations in $v$ is over determined in the case where $c>d$. Assuming there exists a solution $v$ such that Eq.~\ref{eq:Fv} holds, then each $f^{1}_{i}$ is constrained to lie on a $d$ dimensional hyper-ellipse defined by the equation $x^{\top}x+v^{\top}x=0$. Since $v$ is already determined by any $d-1$ columns of $F$, the rest of the columns are restricted to lie on a known ellipse in $\mathbb{R}^d$.
\end{proof}

The situation described in Theorem~\ref{thm:Fv} is even worse for more than two sets of orthogonal weights on top of the representation $D$. The solution in the case of $m$ orthogonal sets would be restricted to lie on the intersection of $\binom{m}{2}$ hyper-ellipses.

The crux of Theorem~\ref{thm:Fv} is the full rank property of $D$. As the theorems below show, if $D$ has $m-1$ low singular values, we can construct solutions with $m$ orthogonal sets of weights that present loss that is only slightly higher than $mL^*(D,y)$. 

Specifically, let $\lambda_1,\lambda_2,...,\lambda_d$ denote the (all non-negative) eigenvalues of the kernel matrix $K=D D^\top$, ordered from largest to smallest. We can bound the loss based on the last eigenvalues. 

\begin{theorem} 
\label{thm:bound1}
There exist sets of weights $F^1=\left[f_1^1,f_2^1,...,f_c^1\right],b^1,F^2=\left[f_1^2,f_2^2,...,f_c^2\right],b^2$ which are orthogonal as follows $\forall j ~~ f_j^1 \perp f_j^2$, for which the joint loss:
\begin{equation} \label{eq:9}
J(F^{1},b^{1},F^{2},b^{2},D,y)=L(F^{1},b^{1},D,y)+L(F^{2},b^{2},D,y)
\end{equation} 
is bounded by
\begin{multline} \label{eq:10}
2L^{*}(D,y)\leq J(F^{1},b^{1},F^{2},b^{2},D,y) \leq 2L^{*}(D,y) + A\lambda_d
\end{multline}
where $A$ is a bounded parameter.
\end{theorem}
\begin{proof}
We prove the theorem by constructing such a solution. Let $v$ be the eigenvector of $K$ corresponding to the smallest eigenvalue $\lambda_d$. We consider the solution $F^1=F^*$,$b^1=b^2=b^*$, $F^2=F^1+v\alpha^\top$, for some vector $\alpha_j = - \frac{||f_j^1||^2}{v^\top f_j^1}$. 

From the construction, it is clear that $L(F^1,b^1,D,y) = L^*(D,y)$ and that the orthogonality constraints $(f_j^1+\alpha_j v)^\top f_j^1=0$ hold for all $j$. 

Let $\Psi = [\psi_1,\psi_2,\hdots,\psi_c]=F^2-F^1$, and let $\psi$ denote the concatination of the column vectors $\psi_j$ into a single column vector.
The expansion of $L(F^1+\Psi,b^1)$ into a multivariate taylor series is as follows:
\begin{equation}
\label{eq:taylor}
L(F^1+\Psi,b^1)=L(F^1,b^1)+(\vec\nabla \cdot \psi)L(D,y)\bigg|_{F^1,b^1}+R(\psi).
\end{equation}
Where $R(\psi)$ represents the remainder term, and can be written in the Lagrange form~\cite{kline1998calculus} as follows:
\begin{equation}
R(\psi)=\frac{1}{2}(\vec\nabla \cdot \psi)^{2}L(D,y)\bigg|_{\rho,b^1}=\frac{\psi^{\top}}{2}\frac{\partial^{2}}{\partial F\partial F}L(D,y)\bigg|_{\rho,b^1}\psi
\end{equation}  
where the derivatives are evaluated at some point $\rho,b^1$ such that $||\rho-F^1||_F\leq||\Psi-F^1||_F$. The first order terms in Eq.~\ref{eq:taylor} vanishes due to the optimality of $F^1,b^1$. Therefore:
\begin{multline} \label{eq:13}
L(F^1+\Psi,b^1) =
L^*(D,y) +\frac{1}{2}\psi^{\top}\frac{\partial^{2}}{\partial F\partial F}L(D,y)\bigg|_{\rho,b^1}\psi
\end{multline}

Using Eq.~\ref{eq:hessian} we can form a bound on the remainder term that does not depend on $\rho$:
\begin{multline} \label{eq:bound}
L(F^1+\Psi,b^1) =L^*(D,y)\\
+\frac{1}{2}\sum_{j=1}^{c}\sum_{j'=j+1}^{c}(\psi_{j}-\psi_{j'})^{\top}\sum_{i=1}^{n} d_{i}d_{i}^{\top}p_{i}(j)p_{i}(j')(\psi_{j}-\psi_{j'})\\
\leq L^*(D,y) + \frac{1}{2}\sum_{j=1}^{c}\sum_{j'=j+1}^{c}(\psi_{j}-\psi_{j'})^{\top}K(\psi_{j}-\psi_{j'})
\end{multline}
Since $\psi_j=\alpha_{j}v$ we get:
\begin{multline} 
L({F}^{1}+\Psi,{B}^{1}) \leq L^*(D,y) +\frac{1}{2}\sum_{j=1}^{c}\sum_{j'=j+1}^{c}(\alpha_{j}-\alpha_{j'})^2 v^{T}Kv \\
= L^*(D,y) +\frac{1}{2}\sum_{j=1}^{c}\sum_{j'=j+1}^{c}(\alpha_{j}-\alpha_{j'})^2 \lambda_d
\end{multline}
Denoting $A=\frac{1}{2}\sum_{j=1}^{c}\sum_{j'=j+1}^{c}(\alpha_{j}-\alpha_{j'})^2$ we have:
\begin{equation} \label{eq:16}
J({F^1},{B^2},{F^1},{B^2},D,y) \leq L^*(D,y)+ A\lambda_{d}
\end{equation}
\end{proof}

Thm.~\ref{thm:bound1} can be generalized to the case of $m$ cross entropy losses as follows.

\begin{theorem}
\label{thm:multiverse}
There exist a set of weights ${F}^1=\left[f_1^1,f_2^1,...,f_C^1\right],{b}^1,{F}^2=\left[f_1^2,f_2^2,...,f_C^2\right],{b}^2...{F}^m=\left[f_1^m,f_2^m,...,f_C^m\right],{b}^m$ which are orthogonal $\forall jrs~~f_j^r \perp f_j^s$  for which the joint loss:
\begin{multline} \label{eq:17}
J(F^{1},b^{1}...F^{m},b^{m},D,y)=\sum_{r=1}^{m}{L(F^{r},b^{r},D,y) }
\end{multline} 
is bounded by:
\begin{multline} \label{eq:18}
mL^{*}(D,y)\leq J(F^{1},b^{1}...F^{m},b^{m},D,y) \\
\leq mL^{*}(D,y)+\sum_{l=1}^{m-1}A_{l}\lambda_{d-j+1}
\end{multline}
where $[A_{1}\hdots A_{m-1}]$ are bounded parameters.
\end{theorem}
\begin{proof}
We again prove the theorem by constructing such a solution. Denoting by $v_{d-m+2}...v_{d}$ the eigenvectors of $K$ corresponding to $\lambda_{d-m+2}\hdots \lambda_{d}$. Given ${F}^{1}=F^*,{b}^{1}=b^*$, we can construct each pair ${F}^{r},{b}^{r}$ as follows:
\begin{multline} \label{eq:19}
\forall j,r ~~~~~~~~{f_j}^{r} = {f_1}^{1}+\sum_{l=1}^{m-1}\alpha_{jlr}v_{d-l+1}\\
{b}^{r}={b}^{1}
\end{multline} 
The tensor of parameters $\alpha_{jlr}$ is constructed to insure the orthogonality condition. Formally, $\alpha_{jlr}$ has to satisfy:
\begin{equation}
\label{eq:tensor}
\forall j,r\neq s ~~~~~~~~ ({f_{j}^1}+\sum_{l=1}^{m-1}\alpha_{jlr}v_{d-l+1})^{\top}{f_{j}^s}=0
\end{equation}
Noticing that $\ref{eq:tensor}$ constitutes a set of $\frac{1}{2}m(m-1)$ equations, it can be satisfied by the tensor $\alpha_{jlr}$ which contains $m(m-1)c$ parameters.
Defining $\Psi^r = [\psi_1^r,\psi_2^r,\hdots,\psi_c^r]=F^r-F^1$, similar to eq ~\ref{eq:bound} we have:
\begin{multline} \label{eq:21}
L({F}^{1}+\Psi^r,{b}^{1}) \leq  L^*(D,y) \\
+ \frac{1}{2}\sum_{j=1}^{c}\sum_{j'=j+1}^{c}(\psi_{j}-\psi_{j'})^{\top}K(\psi_{j}-\psi_{j'})\\
 =  L^*(D,y) \\
+\frac{1}{2}\sum_{j=1}^{c}\sum_{j'=j+1}^{c}\sum_{l=1}^{m-1}(\alpha_{jlr}-\alpha_{j'lr})^{2}v_l^{\top}Kv_l\\
=L^*(D,y) \\
+\frac{1}{2}\sum_{j=1}^{c}\sum_{j'=j+1}^{c}\sum_{l=1}^{m-1}(\alpha_{jlr}-\alpha_{j'lr})^{2}\lambda_{d-l+1}
\end{multline}
Denoting $A_l=\frac{1}{2}\sum_{j=1}^{c}\sum_{j'=j+1}^{c}\sum_{r=1}^{m}(\alpha_{jlr}-\alpha_{j'lr})^{2}$ and summing over all solutions we obtain the bound:
\begin{equation} \label{eq:24}
J(F^{1},b^{1}...F^{m},b^{m},D,y)\leq \sum_{l=1}^{m-1}A_{l}\lambda_{d-l+1}+mL^{*}(D,y)
\end{equation}
We notice that if $\lambda_{d-m+2}=\lambda_{d-m+1}=\hdots\lambda_{d}=0$ then $J(F^{1},b^{1}...F^{m},b^{m},D,y)=mL^{*}(D,y)$.
\end{proof}

\section{Fisher spectrum properties}
\label{sec:fisher}

We next tie the outcome of the multiverse minimization to the Fisher scores used in LDA classification, which served as motivation to our approach.
The Fisher spectrum $\gamma_{1}\hdots \gamma_{d}$ is obtained by solving the generalized eigen-problem $S_b v=\gamma S_w v$, where $S_b$ and $S_w$ are the between class and within class covariance matrices:
\begin{equation} 
S_{b}=\frac{1}{n}\sum_{j=1}^{c}n_j(\mu-\mu_{y_j})(\mu-\mu_{y_j})^\top
\end{equation}
and
\begin{equation} 
S_{w}=\frac{1}{n}\sum_{j=1}^{c}\sum_{i\in I_j}(d_i-\mu_{j})(d_i-\mu_{j})^\top
\end{equation}
where $\mu = \frac{\sum_{i=1}^{n}d_i}{n}$ is the mean of all data points, and $\mu_{j}=\frac{\sum_{i\in I_j}d_i}{n_{j}}$ is the mean of class $j$.
$S_b$ and $S_w$ are the same matrices used in LDA.

The Fisher ratio 
is defined for any vector $v$ as:
\begin{equation}
\label{eq:fr}
\sigma(v,S_b,S_w) = \frac{v^{T}S_{b}v}{v^{T}S_{w}v}
\end{equation}

In the JB formulation, an instance of a class member is influenced by two factors, its class identity and interclass variation. Each class member $d_i$ is modeled as the sum of two Gaussian variables: $d_i=\mu_{y_i}+\epsilon$, 
where $\mu_{y_i}$ is the mean of class $y_i$, and $\epsilon$ represents the intraclass variation. The two terms are modeled as multivariate Gaussians $N(\mathbf{0},S_b)$, $N(\mathbf{0},S_w)$.

Given the above multivariate Gaussian distribution for  $d_i$, the joint distribution $(d_i,d_{i'})$ is also a zero mean multivariate Gaussian. Let $H$ represent the hypothesis that $d_i$ and $d_{i'}$ belong to the same class, and $I$ represent the hypothesis that they belong to different classes. Under the JB formulation, the covariance matrix of the probability distributions $P(d_i,d_{i'}|H)$ and $P(d_i,d_{i'}|I)$ can be derived:
\begin{equation}
\Sigma_H=\begin{pmatrix}
S_b+S_w ~~~~~~~ S_b\\
S_b ~~~~~~~~~ S_b+S_w
\end{pmatrix}, 
\Sigma_I=\begin{pmatrix}
S_b+S_w ~~~~~~~ \mathbf{0}\\
\mathbf{0} ~~~~~~~~~ S_b+S_w
\end{pmatrix}
\end{equation}

Let $\hat{d}=((d_i-\mu)^\top , ~(d_{i'}-\mu)^\top)^\top$.
The log probabilities of the two hypotheses are given, up to a const, by $\hat{d}^\top\Sigma_H^{-1} \hat{d}$ and $\hat{d}^\top\Sigma_I^{-1} \hat{d}$. 
The following theorem links the Fisher spectrum to the success of the JB method.
\begin{theorem} \label{the:jb}
Given data $D$, mean $\mu$ and labels $y$, for any centered data point $\hat{d_i}=d_i-\mu$, we denote $d'_i=(S_b+S_w)\hat{d_i}$. Given two centered data points $\hat{d_1},\hat{d_2}$ such that the fisher ratios $\sigma(d'_1,S_b,S_w),\sigma(d'_2,S_b,S_w)<T$, it holds that:
\begin{multline}
1-2T\leq \frac{\log P(d_1,d_2|H)+\eta_1}{\log P(d_1,d_2|I)+\eta_2} \leq 1+6T  \\ 
\end{multline}
Where $\eta1,\eta2$ are fixed constants.
\end{theorem}
\begin{proof}
In the proof we will be using two matrix inversion identities. The first one is the block matrix inversion identity, for a specific form of block matrices:
\begin{multline}
\label{lm:inv}
\begin{pmatrix}
A~~~~~~~~B \\
B~~~~~~~~A
\end{pmatrix}^{-1}\\
=
\begin{pmatrix}
(A - BA^{-1}B)^{-1}~~~~~~~~-A^{-1}B(A-BA^{-1}B)^{-1} \\
-A^{-1}B(A-BA^{-1}B)^{-1}~~~~~~~~(A - BA^{-1}B)^{-1})
\end{pmatrix}
\end{multline}
The second identity is the Kailath Variant of the Woodbury identity:
\begin{multline}
\label{lm:sum}
(A+BC)^{-1}=A^{-1}-A^{-1}B(\mathbf(I)+CA^{-1}B)^{-1}CA^{-1}
\end{multline}

The proof of Thm.~\ref{the:jb} will also be using the following lemma.
\begin{lemma} \label{lm:inverse}
Let $\gamma_1...\gamma_d$ and $v_1...v_d$ be the generalized eigenvalues and eigenvectors of two positive definite matrices $S_b,S_w$, where $S_b$ is invertible. The spectrum $\gamma_1'...\gamma_d'$ and eigenvectors $v_1'...v_d'$ of the generalized inverse problem  $(S_b+S_w)^{-1}v_i'=\gamma_i'S_b^{-1}v'$ are given by $\gamma' = \frac{\gamma_i}{1+\gamma_i}, v_i' =(S_b+S_w)v_i$, and it holds that $S_b(S_b+S_w)^{-1}v_i' = \gamma_i' v_i'$.
\end{lemma}
\begin{proof}
For the standard generalized problem we have $S_bv_i=\gamma_iS_wv_i$. Therefore, $(S_b+S_w)v_i=(1+\frac{1}{\gamma_i})S_bv_i$. Let $v_i' = (S_b+S_w) v_i$, then 
$v_i'=(1+\frac{1}{\gamma_i})S_b(S_b+S_w)^{-1}v_i'$. Multiplying both sides by $(S_b)^{-1}$ we have $S_b^{-1}v_i'=(1+\frac{1}{\gamma_i})(S_b+S_w)^{-1}v_i'$, and finally $(S_b+S_w)^{-1}v_i'=\frac{\gamma_i}{1+\gamma_i}S_b^{-1}v_i'$. 
\end{proof}

We denote each data point $\hat{d_1}=\sum_{i=1}^{d}\alpha_iv_i',\hat{d_2}=\sum_{i=1}^{d}\beta_i v_i'$ where $v_1'...v_d'$ are the eigenvectors of the generalized inverse eigen-problem $(S_b+S_w)^{-1}v_i'=\gamma_i'S_b^{-1}v_i'$. 
The probabilities $P(d_1,d_2|H),P(d_1,d_2|I)$ are modeled as zero mean gaussian densities with covariances:
\begin{equation}
\Sigma_H = \begin{pmatrix}
S_b+S_w~~~~~~~S_b \\
S_b~~~~~~~S_b+S_w
\end{pmatrix}, \Sigma_I = \begin{pmatrix}
S_b+S_w~~~~~~~~\mathbf{0} \\
\mathbf{0}~~~~~~~~S_b+S_w
\end{pmatrix}
\end{equation}
Denoting $M=(S_b+S_w)^{-1}S_b$ using Eq.~\ref{lm:inv} we have that:
\begin{multline}
\frac{\log P(d_1,d_2|H)+\eta_1}{\log P(d_1,d_2|I)+\eta_2}=\frac{
\begin{pmatrix}
\hat{d_1}^\top ~~ \hat{d_2}^\top  \\
\end{pmatrix}
\Sigma_H^{-1}
\begin{pmatrix}
\hat{d_1}  \\
\hat{d_2}
\end{pmatrix}}{
\begin{pmatrix}
\hat{d_1}^\top ~~ \hat{d_2}^\top  \\
\end{pmatrix}
\Sigma_I^{-1}
\begin{pmatrix}
\hat{d_1}  \\
\hat{d_2}
\end{pmatrix}}\\
=\frac{\hat{d_1}^\top(S_b+S_w-S_bM)^{-1}\hat{d_1}}{\hat{d_1}^\top(S_b+S_w)^{-1}\hat{d_1}+\hat{d_2}^\top(S_b+S_w)^{-1}\hat{d_2}}\\
+\frac{\hat{d_2}^\top(S_b+S_w-S_bM)^{-1}\hat{d_2}}{\hat{d_1}^\top(S_b+S_w)^{-1}\hat{d_1}+\hat{d_2}^\top(S_b+S_w)^{-1}\hat{d_2}}\\
-\frac{\hat{d_1}^\top M(S_b+S_w-S_bM)^{-1}\hat{d_2}}{\hat{d_1}^\top(S_b+S_w)^{-1}\hat{d_1}+\hat{d_2}^\top(S_b+S_w)^{-1}\hat{d_2}}\\
-\frac{\hat{d_2}^\top M(S_b+S_w-S_bM)^{-1}\hat{d_1}}{\hat{d_1}^\top(S_b+S_w)^{-1}\hat{d_1}+\hat{d_2}^\top(S_b+S_w)^{-1}\hat{d_2}}
\end{multline}
where the constants of the densities have been expressed by $\eta_1,\eta_2$ in the left hand side of the equation. Defining $M' =(S_b+S_w)^{-1}S_b(S_b+S_w)^{-1}$ and $S=M[\mathbf{I}+M^2]^{-1}$ by using Eq.~\ref{lm:sum}:
\begin{multline}
(S_b+S_w-S_bM)^{-1} = SM'+(S_b+S_w)^{-1}
\end{multline}
Therefore:
\begin{multline} \label{eq:long}
\frac{log(d_i,\mu|H) - \eta_1}{log(d_i,\mu|I)- \eta_2} =1\\
+\frac{\hat{d_1}^\top SM'\hat{d_1}+\hat{d_2}^\top SM'\hat{d_2}-\hat{d_1}^\top MSM'\hat{d_2}-\hat{d_2}^\top MSM'\hat{d_1}}{\hat{d_1}^\top(S_b+S_w)^{-1}\hat{d_1}+\hat{d_2}^\top(S_b+S_w)^{-1}\hat{d_2}}\\
-\frac{2\hat{d_1}^\top M'\hat{d_2}}{\hat{d_1}^\top(S_b+S_w)^{-1}\hat{d_1}+\hat{d_2}^\top(S_b+S_w)^{-1}\hat{d_2}}
\end{multline}
Defining $\rho_i=\frac{\gamma_i}{1+\gamma_i}$ We notice from Lemma~\ref{lm:inverse} that $v_i'^\top M =\gamma_i v_i'^\top$, $v_i'^\top S=v_i'^\top M[\mathbf{I}+M^2]^{-1}=\frac{\rho_i}{1+\rho_i^2}v_i'^\top$ And so we can expand the first term in the numerator of Eq.~\ref{eq:long}:
\begin{multline}
\hat{d_1}^\top SM'\hat{d_1}=(\sum_{i=1}^{k}\alpha_i v_i')^\top SM'(\sum_{i=1}^{k}\alpha_i v_i')\\
=(\sum_{i=1}^{k}\alpha_i \frac{\rho_i}{1+\rho_i^2} v_i')^\top M'(\sum_{i=1}^{k}\alpha_i v_i')\\
=\sum_{i=1}^{k}\alpha_i \frac{\rho_i}{1+\rho_i^2} v_i'^\top (S_b+S_w)^{-1}S_b(S_b+S_w)^{-1}\sum_{i=1}^{k}\alpha_i v_i'
\end{multline}

Since $v_i'=(S_b+S_w)v_i,~~ \forall i\neq j~~v_i^\top S_b v_j=0$, and $\rho>0$ we get:
\begin{multline}
\hat{d_1}^\top SM'\hat{d_1}=(\sum_{i=1}^{k}\alpha_i \frac{\rho_i}{1+\rho_i^2} v_i)^\top S_b(\sum_{i=1}^{k}\alpha_i v_i)\\
=\sum_{i=1}^{k}\alpha_i^2 \frac{\rho_i}{1+\rho_i^2} v_i^\top S_b v_i\leq \sum_{i=1}^{k}\alpha_i^2 v_i^\top S_b v_i=d_1'^\top S_b d_1'
\end{multline}
Since it is also true that $\forall i\neq j~~v_i^\top (S_b+S_w) v_j=0$, similar manipulation can be done with the denominator:
\begin{multline}
\hat{d_1}^\top (S_b+S_w)^{-1}\hat{d_1}+\hat{d_2}^\top (S_b+S_w)^{-1}\hat{d_2}\\
=\sum_{i=1}^{k}\alpha_i  v_i^\top (S_b+S_w)\sum_{i=1}^{k}\alpha_i v_i\\
+\sum_{i=1}^{k}\beta_i  v_i^\top (S_b+S_w)\sum_{i=1}^{k}\beta v_i\\
=d_1'^\top (S_b+S_w)d_1' +d_2'^\top (S_b+S_w)d_2'
\end{multline}
And so:
\begin{multline}
0 < \frac{\hat{d_1}^\top SM'\hat{d_1}+\hat{d_2}^\top SM'\hat{d_2}}{\hat{d_1}^\top (S_b+S_w)^{-1}\hat{d_1}+\hat{d_2}^\top (S_b+S_w)^{-1}\hat{d_2}}\\
\leq \frac{d_1'^\top S_b d_1'+d_2'^\top S_b d_2'}{d_1'^\top (S_b+S_w)d_1' +d_2'^\top (S_b+S_w)d_2'}\leq 2T
\end{multline}

The last reasoning stems from the bound on fisher scores of $d_1$ and $d_2$ and the fact that if both $\frac{a_1}{a_2}$ and $\frac{b_1}{b_2}$ are smaller than $T$ and all terms are positive, then $\frac{a_1+b_1}{a_2+b_2}$ is smaller than both $\frac{2*\max(a_1,b_1)}{a_2}$ and $\frac{2*\max(a_1,b_1)}{b_2}$, and therefore $\frac{a_1+b_1}{a_2+b_2}<2T$.
The same manipulations to the rest of the terms in the numerator of Eq.~\ref{eq:long} and get the following bound:
\begin{multline}
\left| \frac{-\hat{d_1}^\top MSM'\hat{d_2}-\hat{d_2}^\top MSM'\hat{d_1}-2\hat{d_1}^\top M'\hat{d_2}}{\hat{d_1}^\top(S_b+S_w)^{-1}\hat{d_1}+\hat{d_2}^\top(S_b+S_w)^{-1}\hat{d_2}}\right|<4T~~,
\end{multline}
from which the theorem stems.
\end{proof}
Theorem ~\ref{the:jb} indicates that in the directions of low Fisher ratio the JB method cannot distinguish between the two competing hypotheses and determine whether the two samples $d_i$ and $d_{i'}$ belong to the same class.


We observed during experiments performed on a number of datasets, that  training of a CNN using a single cross entropy loss produces a representation that has a rapidly decreasing Fisher spectrum, and is highly discriminative in only a few directions. Reducing the representation dimension, i.e., using a bottleneck technique helps in reducing the total number of dimensions but does not seem to increase the number of discriminative dimensions. We next show that by optimizing for multiple orthogonal solutions, we promote more directions that have high Fisher scores. 

Since the hyperplanes $f_i^r$ learned during optimization are discriminative, we can expect most of these to have high Fisher ratios. The multiplicity created by the multiverse loss, leads to multiple orthogonal hyperplanes. Since the probabilities produced by the matching hyperplanes are identical, it is likely that all matching hyperplanes $f_j^r$, and $f_j^s$ have similar Fisher ratios. The theorem below shows that adding more $S_w$-orthogonal classifiers with high Fisher ratios increases the $L1$ norm of the Fisher spectrum.

\begin{theorem} \label{thr:mf}
Let $f^1...f^m$ be a set of $m$ classifiers that are $S_w$-orthogonal for data $D$ and labels $y$, and let $\gamma=[\gamma_1...\gamma_d]$ denote the Fisher spectrum. 
Given that $\forall 1\leq r \leq m$, for some value $\theta$, $\sigma(f^r,S_b,S_w)\geq \theta$,  it holds that $\sum_{k=1}^{d} \gamma_k \geq \sqrt{m}\theta$ .
\end{theorem}
\begin{proof}
The Fisher spectrum $\gamma_1...\gamma_d$ is obtained from the eigenvalues of $R=S_b^{\frac{1}{2}}S_w^{-1}S_b^{\frac{1}{2}}$.
For each classifier $f^r$ we have:
\begin{equation}
\forall r,~~\frac{f^{r\top}S_b f^r}{f^{r\top}S_w f^r}\geq \theta
\end{equation}
Denoting $u^r = \frac{S_w^{\frac{1}{2}}f^r}{||S_w^{\frac{1}{2}}f^r||}$, we have:
\begin{equation}
\forall r,~~u^{r\top} S_w^{-\frac{1}{2}}S_b S_w^{-\frac{1}{2}}u^r = u^{r\top} \hat{R} u^r\geq \theta
\end{equation}
Denoting by $\hat{\gamma}=[\hat{\gamma_1}...\hat{\gamma_d}]$ and $w_1...w_d$ the eigenvalues and eigenvectors of $\hat{R}$, we notice that $\sum_{k=1}^{d}\hat{\gamma_k}=\sum_{k=1}^{d}\gamma_k$, since the matrix $\hat{R}$ is a cyclic permutation of $R=S_b^{\frac{1}{2}}S_w^{-1}S_b^{\frac{1}{2}}$, and hence have equal trace. The eigenvectors of $\hat{R}$ span a $d$ dimensional linear subspace, and so we can express each $u^r=\sum_{k=1}^{d}\alpha^r_k w_k, ||\alpha^r||_2=1$. From the $S_w$ orthogonality property of the solutions $f^1...f^m$, it follows that $\forall r\neq s, ~~u^{r\top} u^s = \alpha^{r\top}\alpha^s=0$  We therefor have:
\begin{equation}
\forall r,~~\sum_{k=1}^{d}\alpha^r_k w^{\top}_k \hat{R} \sum_{k=1}^{d}\alpha^r_k w_k = \sum_{k=1}^{d}(\alpha^r_k)^2\hat{\gamma_k}.
\end{equation}
In matrix form:
\begin{multline} \label{eq:mat}
Diag(\begin{pmatrix}
\alpha^1_1...\alpha^1_d  \\
\alpha^2_1...\alpha^2_d\\
\vdots  \\
\alpha^r_1...\alpha^r_d\\
\end{pmatrix}
\begin{pmatrix}
\hat{\gamma_1}~~~ 0~~~...0 \\
0 ~~~\hat{\gamma_2}~~~ 0...0\\
\vdots  \\
0~~~0~~~...~~~\hat{\gamma_d}\\
\end{pmatrix}
\begin{pmatrix}
\alpha^1_1...\alpha^1_d  \\
\alpha^2_1...\alpha^2_d\\
\vdots  \\
\alpha^r_1...\alpha^r_d\\
\end{pmatrix}^{\top}) \\
= Diag(\Delta \Gamma \Delta^{\top})\geq\begin{pmatrix}
\theta \\
\theta\\
\vdots  \\
\theta\\
\end{pmatrix}\\
\end{multline}
And hence:
\begin{equation} \label{eq:tr}
tr(\Delta \Gamma \Delta^{\top})=tr( \Delta^{\top}\Delta\Gamma)\geq m\theta
\end{equation}
We notice that $\Delta^{\top}\Delta$ is positive semi definite, and from the orthonormality of the vectors $a^1...a^m$ we have that $(\Delta^{\top}\Delta)^2=\Delta^{\top}\Delta$, $tr(\Delta^{\top}\Delta)=m$. The Cauchy-Schwartz inequality states that for any two positive semi definite matrices of the same size $X,Y$, it holds that $tr(XY)\leq \sqrt{tr(X^2)tr(Y^2)}$, and so we have:
\begin{multline}
m\theta\leq tr( \Delta^{\top}\Delta\Gamma)\leq \sqrt{tr((\Delta^{\top}\Delta)^2)tr(\Gamma^2)}\\
=\sqrt{tr(\Delta^{\top}\Delta)\sum_{i=1}^{d}\hat{\gamma_i}^2}=\sqrt{m}\sqrt{\sum_{i=1}^{d}\hat{\gamma_i}^2}\leq \sqrt{m}\sum_{i=1}^{d}\hat{\gamma_i}\\
=\sqrt{m}\sum_{i=1}^{d}\gamma_i
\end{multline}
And so finally:
\begin{equation}
m\theta \leq \sqrt{m}\sum_{i=1}^{d}\gamma_i \rightarrow \sum_{i=1}^{d}\gamma_i\geq \sqrt{m}\theta
\end{equation}
\end{proof}

In Thm.~\ref{thr:mf} we used the $S_w$ orthogonality of the solutions to guarantee the result, however it is not a necessary condition. From our experiment we noticed an improved Fisher spectrum when both $S_w$ and the standard orthogonality condition were used.

\section{Experiments}

In order to evaluate the effect of using the multiverse loss on performance, we have conducted experiments on two widely used datasets: CIFAR-100 and LFW. While the CIFAR-100 experiments are performed using a new transfer learning protocol, the LFW experiments provide a direct empirical comparison to a large body of previous work.

\subsection{Network architecture}

In our experiments, we employ two network architectures. For the CIFAR-100 experiments, we use the architecture of network in network~\cite{nin}; for the face recognition experiments, we use the scratch architecture~\cite{casia}. The networks were trained from scratch at each experiment, using the MatConvNet framework~\cite{vedaldi15matconvnet}.

Both networks are fully convolutoinal, and we added a hidden layer on top of the networks in order  to apply our method on top of a vector of activations. This modification is not strictly needed and was made for implementation convenience. This top layer was used as the representation. The architectures used are given, for completeness, in Tab.~\ref{tab:nin} and Tab.~\ref{tab:scratch} for the network in network and scratch networks respectively.

\begin{table}[t!]
\begin{center}
\begin{tabular}{|l|c|c|c|c|}
\hline
Layer           & Filter/Stride & \#Channel & \#Filter    \\ \hline
Conv11        & $5 \times 5 $ / 1         & 3           & 192                             \\ \hline
Conv12        & $1 \times 1 $ / 1         & 192          & 160                          \\ \hline
Conv13        & $1 \times 1$ / 1        & 160          & 96                               \\ \hline
Pool1        & $3 \times 3 $ / 2         & 96          & --                             \\ \hline
Dropout1-0.5       & --          & --          & --                             \\ \hline
Conv21        & $5 \times 5 $ / 1         & 96           & 192                             \\ \hline
Conv22        & $1 \times 1 $ / 1         & 192          & 192                          \\ \hline
Conv23        & $1 \times 1$ / 1        & 192          & 100                               \\ \hline
Pool2        & $3 \times 3 $ / 2         & 192   
& -- 
\\ \hline
Dropout1-0.5       & --          & --          & --                             \\ \hline
Conv31        & $3 \times 3 $ / 1         & 192           & 192                            \\ \hline
Conv32        & $1 \times 1 $ / 1         & 192          & 192                           \\ \hline
Conv33        & $1 \times 1$ / 1        & 192          & 100                              \\ \hline
Avg Pool        & $7 \times 7 $ / 1     & 100  
&--
\\ \hline
FC        & $1 \times 100$ / 1        & 100          & 100                              \\ \hline
\end{tabular}
\end{center}
\caption{The modified NIN~\cite{nin} model used in the CIFAR-100 experiments. The network starts with a color input image of size $3 \times 32 \times 32$ pixels, and runs through 3 convolutional blocks interleaved with ReLU and max pooling layers. Following a spatial average pooling at the end of the process, a representation of size 100 is obtained. A FC layer of size 100 was added to the architecture for reasons of implementation convenience.}
\label{tab:nin}
\end{table}

\begin{table}[t!]
\begin{center}
\begin{tabular}{|l|c|c|c|c|}
\hline
Layer           & Filter/Stride & \#Channel & \#Filter \\ \hline
Conv11        & $3 \times 3 $ / 1         & 1           & 32                             \\ \hline
Conv12        & $3 \times 3 $ / 1         & 32          & 64                           \\ \hline
Max Pool        & $2 \times 2$ / 2        & 64          & --                               \\ \hline
Conv21        & $3 \times 3 $ / 1         & 64          & 64                              \\ \hline
Conv22        & $3 \times 3 $ / 1         & 64          & 128                             \\ \hline
Max Pool        & $2 \times 2$ / 2        & 128         & --                \\ \hline
Conv31        & $3 \times 3 $ / 1         & 128         & 96                              \\ \hline
Conv32        & $3 \times 3 $ / 1         & 96          & 192                            \\ \hline
Max Pool        & $2 \times 2$ / 2        & 192         & --                             \\ \hline
Conv41        & $3 \times 3 $ / 1         & 192         & 128                            \\ \hline
Conv42        & $3 \times 3 $ / 1         & 128         & 256                           \\ \hline
Max Pool        & $2 \times 2$ / 2        & 256         & --                                \\ \hline
Conv51        & $3 \times 3 $ / 1         & 256         & 160                             \\ \hline
Conv52        & $3 \times 3 $ / 1         & 160         & 320                             \\ \hline
Avg Pool        & $6 \times 6$ / 1        & 320         & --                             \\ \hline
Dropout1-0.3       & --          & --          & --                             \\ \hline
FC        & $1 \times 320$ / 1        & 320          & 100                              \\ \hline
\end{tabular}
\end{center}
\caption{The scratch model by the authors of~\cite{casia}, which is the face recognition network in our experiments. The network starts with a gray scale input image of size $1 \times 100 \times 100$ pixels, and runs through 10 convolutional layers interleaved ReLU and max pooling layers. Following a spatial average pooling at the end of the process, a representation of size 320 is obtained. A FC layer of size 320 was added to the scratch architecture for reasons of implementation convenience.}
\label{tab:scratch}
\end{table}

\subsection{Results}
The CIFAR-100~\cite{cifar10or100} contains 50,000 $32 \times 32$ color images, split between 100 categories. The images were extracted from the tiny image collection~\cite{eightymil}. Throughout our experiments, the first 90 classes (class ids 0 to 89) are used as the source domain, and the last 10 as the target domain.


Our experiments compare six architectures: a baseline with one cross entropy loss (``M1''); four multiverse architectures with 2--5 such losses (``M2--M5''); and an ensemble of five networks with a single cross entropy loss each. The last method was added to demonstrate that our method's benefit is greater than that of combining multiple networks. Note, however, that when compounding losses, the overall network architecture resembles that of a single network and is almost as efficient to train and deploy as the baseline network. One can easily create ensembles of networks with multiverse losses, as we do for the LFW benchmark.

We report the methods' performance in multiple ways. The validation error reports the error rate obtained, in the source domain of 90 classes, on the 10\% of the data reserved for this purpose. In the target domain, two metrics are used: same/not-same accuracy using either the cosine distance or the JB method. Note that the cosine distance is unsupervised, and that we train the JB on the validation set of the source domain. Hence, no training was done in the target domain. For the same/not-same evaluation, 3000 matching and 3000 non-matching pairs were randomly sampled from the 10 classes of the target domain.

As can be seen in Tab.~\ref{tab:cifar100_table}, the multiverse method outperform the baseline and the ensemble methods on the target domain, in each of the accuracy metrics. It is also evident that adding more cross entropy losses improves performance. The preferable separation between the classes is also depicted visually in Fig.~\ref{fig:tsne}, where the 2D embedding of the baseline (M1) representation is compared to that obtained using the M5 multiverse method. For the purpose of this visualization, the TSNE~\cite{tsne} embedding method is used. 

\begin{table}[t]
\begin{center}
\begin{tabular}{|l|c|c|c|}
\hline
Domain & Source  & \multicolumn{2}{c|}{Target (transfer)} \\
\hline
Metric          & Val error          & Cosine    & JB \\
\hline
M1          & 0.340              & 0.789             & 0.800                          \\
M2          & 0.340            & 0.791            & 0.804                         \\
M2 ($S_w$-orthogonal)         & 0.344            & 0.798             & 0.803                         \\
M3          & 0.345            & 0.801             & 0.812                   \\
M3 ($S_w$-orthogonal)    & 0.346            & 0.799             & 0.811                   \\
M4          & 0.351            & 0.807             & 0.82                      \\
M4 ($S_w$-orthogonal)    & 0.353            & 0.808             & 0.823                      \\
M5          & 0.360           & 0.812             & 0.833                      \\
M5 ($S_w$-orthogonal)    & 0.362            & 0.811             & 0.831                      \\
M6     & 0.369            & 0.816             & 0.838                      \\
M6 ($S_w$-orthogonal)    & 0.371            & 0.816             & 0.834                      \\
M7     & 0.375            & 0.815             & 0.831                      \\
M7 ($S_w$-orthogonal)    & 0.377            & 0.816             & 0.830                      \\
\hline
\hline
Ensemble of 5 times M1  & NA & 0.803 & 0.82  \\
\hline
\end{tabular}
\end{center}
\caption{CIFAR-100 Results. Multiverse networks of multiplicity 1--7 are shown, for both types of orthogonality. Also shown is the result obtained by an ensemble of 5 conventional networks. The numbers indicate either the validation error or the same/not-same accuracy in the target domain.}
\label{tab:cifar100_table}
\end{table}

\begin{figure}
\centering
\begin{tabular}{c}
\includegraphics[width=.85\linewidth]{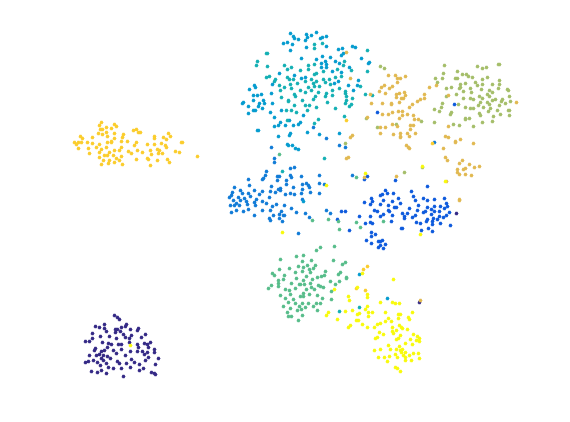}\\
(a)\\
\includegraphics[width=.85\linewidth]{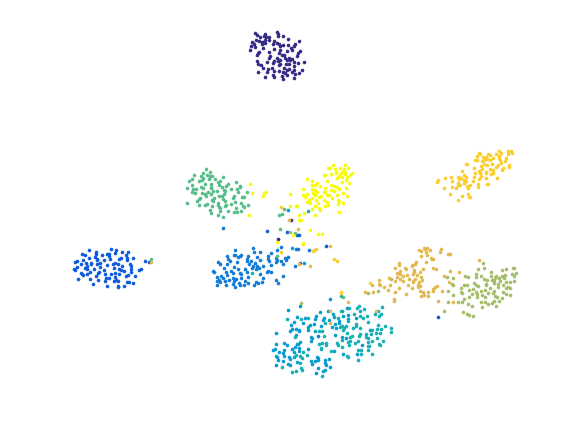}\\
(b)\\
\end{tabular}
\caption{2D embedding (TSNE~\cite{tsne}) of the representation of (a) conventional cross entropy and (b) multiverse (M5). The 10 target classes of the CIFAR-100 experiment are shown.}
\label{fig:tsne}
\end{figure}

\begin{table*}[t]
\begin{center}
\begin{tabular}{|l|c|c|c|c|c|}
\hline
Domain & Source & \multicolumn{3}{c|}{Target (transfer)} \\
\hline
Metric          & Val error          & Cosine    & JB on source & JB on LFW splits       \\
\hline 
CASIA trained M1          & 0.07              & 0.962  $\pm$ 0.0032         & 0.966 $\pm$ 0.0022&  0.970   $\pm$ 0.0016     \\
CASIA trained M1 (2)      & 0.07              & 0.962  $\pm$ 0.0021        & 0.966 $\pm$ 0.0019& 0.971  $\pm$ 0.0022     \\
CASIA trained M1 (3)      & 0.07              & 0.961 $\pm$ 0.0022           & 0.966 $\pm$ 0.0013&  0.971  $\pm$  0.0015    \\
Ensemble of 3 CASIA M1 & & 0.968 $\pm$ 0.0019& 0.972 $\pm$ 0.0021& 0.975 $\pm$ 0.0025\\
\hline
CASIA trained M2          & 0.08           & 0.970 $\pm$ 0.0021           & 0.974  $\pm$ 0.0017 &  0.976 $\pm$ 0.0016       \\
CASIA trained M3          & 0.11            & 0.972  $\pm$ 0.0012          & 0.977 $\pm$ 0.0015 &  0.980  $\pm$  0.0034     \\
CASIA trained M3 (2)           & 0.11            & 0.971  $\pm$ 0.0031          & 0.977 $\pm$ 0.0028 &  0.979 $\pm$ 0.0027       \\
CASIA trained M5 (1)           & 0.12            & 0.973 $\pm$ 0.0011           & 0.978 $\pm$ 0.0014 &0.981 $\pm$ 0.0019       \\
CASIA trained M5 (2)           & 0.12            & 0.972  $\pm$ 0.0015         & 0.977 $\pm$ 0.0019 &  0.980  $\pm$ 0.0031      \\
\hline
3rd party DB, M5           & 0.12            & 0.982  $\pm$ 0.0034          & 0.982 $\pm$ 0.0031 &  0.988 $\pm$ 0.0035       \\
\hline
Two network ensemble        &             & 0.985 $\pm$ 0.0029            & 0.990 $\pm$ 0.0027 &  0.991 $\pm$ 0.0027       \\

\hline
\end{tabular}
\end{center}
\caption{Face recognition results. Shown are the validation error on CASIA, and transfer results on LFW. The cosine similarity as well as learned JB similarities are shown. The JB was either trained on CASIA or on the LFW training splits. The LFW results confirm with the unrestricted mode and report mean and Standard Error of the accuracy obtained for the ten cross validation splits.}
\label{tab:casia_table}
\end{table*}

\begin{table*}[t]
\begin{center}
\begin{tabular}{|l|c|c|c|l|}
\hline
Method                                          & Single network  & Ensemble result & \#nets & Training dataset     \\
\hline
M5                                              &      0.9814 $\pm$ 0.0019            & --              &                                   & CASIA~\cite{casia}                \\
M5, 3rd party DB                                & 0.9883 $\pm$ 0.0035 & 0.9905 + 0.0027 & 2                                 & proprietary 800k images  \\
DeepFace~\cite{deepface}                      & 0.9700 $\pm$ 0.0087   & 0.9735 $\pm$ 0.0025 & 7                           & proprietary, 4M images              \\
DeepID~\cite{deepid1}                         & --              & 0.9745 $\pm$ 0.0026 & 25                           & proprietary,160k                    \\
Original scratch~\cite{casia}               & 0.9773 $\pm$ 0.0031 & --              & 1                            & CASIA~\cite{casia}                \\
Web-Scale Training~\cite{Taigman_2015_CVPR} & 0.9800   & 0.9843    & 4                            & proprietary, 500M images            \\
MSU TR~\cite{jain}                            & 0.9745 $\pm$ 0.0099 & 0.9823 $\pm$ 0.0068 & 7                       & CASIA~\cite{casia}                \\
MMDFR~\cite{c67}                              & 0.9843 $\pm$ 0.0020 & 0.9902 $\pm$ 0.0019 & 8                        & proprietary,500k                    \\
DeepID2~\cite{deepid2}                        & 0.9633          & 0.9915 $\pm$ 0.0013 & 25                    & proprietary,160k                    \\
DeepID2+~\cite{deepid2plus}                   & 0.9870            & 0.9947 $\pm$ 0.0012 & 25                     & proprietary,290k                    \\
FaceNet~\cite{facenet}                            & 0.9887 $\pm$ 0.0015 & 0.9963 $\pm$ 0.0009 & 8                   & proprietary, 200M                   \\
\hline
FR+FCN~\cite{c45}(*)                             & --              & 
0.9645 $\pm$ 0.0025 & 5                        & CelebFaces~\cite{celebfaces}, 88k \\
betaface.com(*)                                    & --              & 0.9808 $\pm$ 0.0016 & NA                      & NA                                  \\
Uni-Ubi(*)                                         & --              & 0.9900 $\pm$ 0.0032 & NA                       & NA                                  \\
Face++~\cite{c40}(*)                             & --              & 0.9950 $\pm$ 0.0036 & 4                       & proprietary, 5M face images         \\
DeepID3~\cite{deepid3}(*)                        & --              & 0.9953 $\pm$ 0.0010 & 25                    & proprietary,300k                    \\
Tencent-BestImage(*)                               & --              & 0.9965 $\pm$ 0.0025 & 20                      & proprietary, 1M face images         \\
Baidu~\cite{baidu}(*)                            & --              & 0.9977 $\pm$ 0.0006 & 10                      & proprietary, 1.2M face images       \\
AuthenMetric(*)                       & --              & 0.9977 $\pm$ 0.0009 & 25                       & proprietary, 500k face images \\     
\hline
\end{tabular}
\end{center}
\caption{Comparison to state of the art results on LFW. Out of all the methods trained on CASIA, we present the best performance. We also present the best result for a single network, with the exception of FaceNet, which was trained on a dataset which is a hundred times larger than ours. A star (*) indicates commercial systems whose claimed results were not peer reviewed.}
\label{tab:lfwresults}

\end{table*}
\begin{figure*}[!t]
\begin{center}
\begin{tabular}{ccc}
\includegraphics[width=.2879\linewidth]{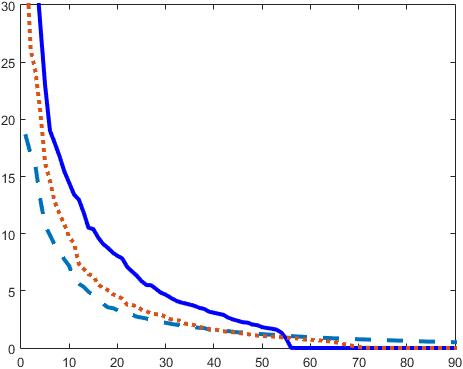}&
\includegraphics[width=.2879\linewidth]{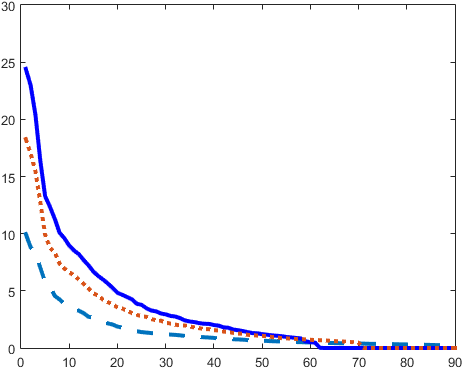}&
\includegraphics[width=.2879\linewidth]{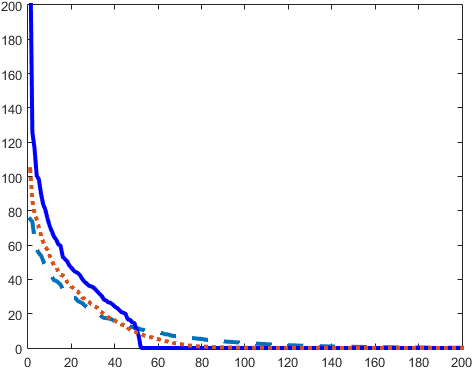}\\
(a) & (b) & (c)\\
\includegraphics[width=.2879\linewidth]{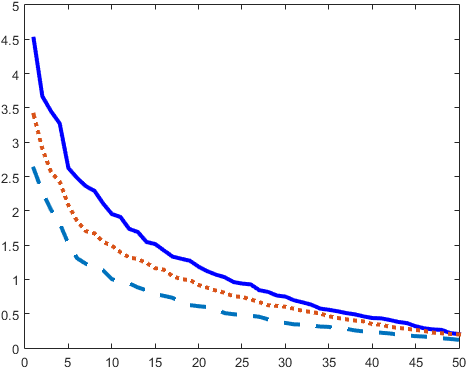}&
\includegraphics[width=.2879\linewidth]{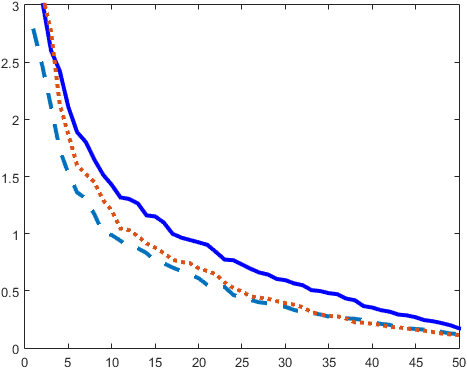}&
\includegraphics[width=.2879\linewidth]{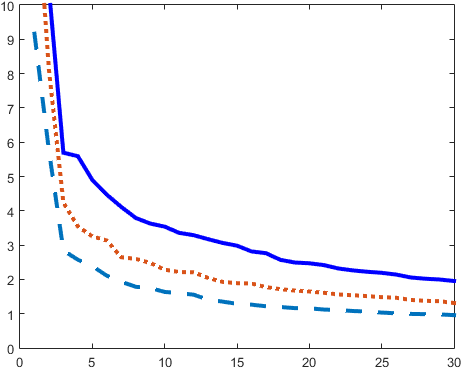}\\
(d) & (e) & (f)\\
\end{tabular}
\end{center}
\caption{Various spectrums obtained in the target domain. Each plot shows singular values (first row) or generalized eigenvalues (second row), sorted separately for each of three methods. Solid blue is the result obtained for M5. The dotted red line is the M3 result, and the baseline M1 is shown as dashed green. (a-c) For CIFAR-100 employing conventional orthogonality, CIFAR-100 $S_w$ orthogonality, and LFW, respectively, the singular values of the kernel matrix $K=DD^\top$ are shown. The multiverse loss leads to higher values until it drops to zero earlier than the conventional spectrum. (d-f) The Fisher spectrum, i.e., the generalized eigenvalues of $S_w$ and $S_b$ for the same three datasets: (d) CIFAR conventional orthogonality, (e) CIFAR $S_w$ orthogonality and (f) LFW. As a result of applying the multiverse method, there is an increase in the magnitude of the eigenvalues.}
\label{fig:spectrums}
\end{figure*}

As mentioned above, for the face recognition experiments, we use the scratch model~\cite{casia}. 
The networks are trained on the CASIA dataset~\cite{casia}; LFW dataset~\cite{lfw} is used as the target domain. 

Models are evaluated in the source domain by measuring the classification accuracy on the CASIA dataset, which we split to 90\% training and 10\% validation. For the target domain, the LFW benchmark in the unrestricted mode~\cite{newlfw} is used (we do not use person ID from LFW, but do use the IDs of the CASIA dataset). The LFW results are mean and Standard Error estimated over the fixed ten cross-validation splits. JB is either trained on the CASIA validation split or on the LFW dataset itself in a cross validation manner.

In the LFW experiments, we performed the M1 (baseline), M3, and M5 experiments multiple times, in order to show the stability of the results and to support ensembles. The $S_w$-orthogonality multiverse method, which is slower to train, was not tested on LFW by the submission date. As can be seen in Tab.~\ref{tab:casia_table}, the multiverse loss outperforms, in the target domain, the baseline method and also outperforms the ensemble of multiple baseline networks. This is true for the cosine similarity, as well as for the two JB experiments. Interestingly, multiverse  does not show an advantage in the source domain (this does not weaken our claims).

In face recognition, the effect of the training dataset sometimes overshadows that of the method. We, therefore, employed a proprietary 800k images 3rd party dataset, which does not intersect the identities of the LFW dataset. In comparison to CASIA's 500k images, the 3rd party dataset is slightly larger and contains fewer tagging mistakes. As can be seen in Tab.~\ref{tab:casia_table}, this leads to an improvement in performance. By combining two networks (i) the M5 network trained on this outside dataset and (ii) the M5 network trained on CASIA, we are able to further improve results.

The results we obtained are compared in Tab.~\ref{tab:lfwresults} to the state of the art as reported on the LFW webpage on the date of the submission. Our results, which use a fairly simple fully convolutional architecture outperform all CASIA trained networks. In addition, the ranking obtained for a single network outperforms all results, except one result~\cite{facenet}, which was obtained using 200 million images. 

In addition to performance, we also examined the effect of the multiverse loss on the properties of the representation. Fig.~\ref{fig:spectrums} demonstrate the singular values of the data representation in the transfer domain on (a) CIFAR-100 using conventional orthogonality, (b) CIFAR-100 with $S_w$-orthogonality, and (c) LFW. As can be seen, the multiverse network (M5) has larger singular values. However, these drop to zero abruptly whereas the spectrum of the baseline representation continues to decay gradually. As a result, the representation of our method is of a lower dimension, and is more balanced among the dimensions. Fig.~\ref{fig:spectrums} (d), (e), and (f) show the generalized eigenvalues of $S_b$ and $S_w$ in the target domain. As can be seen, the multiverse method promotes larger Fisher ratios. 

The sharp drop in the data dimensionality that is promoted by the multiverse method leads to very compact representations. The dimensionality of our best single network (M5, 3rd party dataset), is only 51 (Fig.~\ref{tab:lfwresults}(c)). This is a very compact representation, which is much lower than any other state of the art network. 

\section{Conclusions}

This work presented the emergence of surprising and desirable properties of the representation layer of a deep neural network when learning multiple orthogonal solutions. The practical implications of our work are far reaching since the suggested method is easy to incorporate into almost any architecture.\

{\small
\bibliographystyle{ieee}
\bibliography{cvpr2016,transferable}

\begin{thebibliography}{10}\itemsep=-1pt

\bibitem{brown2005diversity}
G.~Brown, J.~Wyatt, R.~Harris, and X.~Yao.
\newblock Diversity creation methods: a survey and categorisation.
\newblock {\em Information Fusion}, 6(1):5--20, 2005.

\bibitem{jb}
D.~Chen, X.~Cao, L.~Wang, F.~Wen, and J.~Sun.
\newblock Bayesian face revisited: A joint formulation.
\newblock In {\em European Conf. Computer Vision}, 2012.

\bibitem{Cover:2006:EIT:1146355}
T.~M. Cover and J.~A. Thomas.
\newblock {\em Elements of Information Theory (Wiley Series in
  Telecommunications and Signal Processing)}.
\newblock Wiley-Interscience, 2006.

\bibitem{imagenet}
J.~Deng, W.~Dong, R.~Socher, L.-J. Li, K.~Li, and L.~Fei-Fei.
\newblock {ImageNet: A Large-Scale Hierarchical Image Database}.
\newblock In {\em CVPR09}, 2009.

\bibitem{c67}
C.~Ding and D.~Tao.
\newblock Robust face recognition via multimodal deep face representation.
\newblock {\em CoRR}, abs/1509.00244, 2015.

\bibitem{FeiFei06}
L.~Fei-Fei, R.~Fergus, and P.~Perona.
\newblock One-shot learning of object categories.
\newblock {\em Pattern Analysis and Machine Intelligence, IEEE Transactions
  on}, 28(4):594--611, 2006.

\bibitem{Girshick14}
R.~Girshick, J.~Donahue, T.~Darrell, and J.~Malik.
\newblock Rich feature hierarchies for accurate object detection and semantic
  segmentation.
\newblock In {\em Computer Vision and Pattern Recognition (CVPR), 2014 IEEE
  Conference on}, pages 580--587. IEEE, 2014.

\bibitem{fastrcnn}
R.~B. Girshick.
\newblock Fast {R-CNN}.
\newblock {\em CoRR}, abs/1504.08083, 2015.

\bibitem{wccn}
A.~O. Hatch, S.~Kajarekar, and A.~Stolcke.
\newblock Within-class covariance normalization for svm-based speaker
  recognition.
\newblock In {\em Proc. of ICSLP}, page 14711474, 2006.

\bibitem{newlfw}
G.~B. Huang and E.~Learned-Miller.
\newblock Labeled faces in the wild: Updates and new reporting procedures.
\newblock UM-CS-2014-003, 2014.

\bibitem{lfw}
G.~B. Huang, M.~Ramesh, T.~Berg, and E.~Learned-Miller.
\newblock Labeled faces in the wild: A database for studying face recognition
  in unconstrained environments.
\newblock Technical Report 07-49, University of Massachusetts, Amherst, October
  2007.

\bibitem{Karpathy}
A.~Karpathy and L.~Fei-Fei.
\newblock Deep visual-semantic alignments for generating image descriptions.
\newblock In {\em Computer Vision and Pattern Recognition (CVPR)}, 2015.

\bibitem{kline1998calculus}
M.~Kline.
\newblock {\em Calculus: an intuitive and physical approach}.
\newblock Courier Corporation, 1998.

\bibitem{cifar10or100}
A.~Krizhevsky.
\newblock {Learning Multiple Layers of Features from Tiny Images}.
\newblock Master's thesis, 2009.

\bibitem{Krizhevsky12}
A.~Krizhevsky, I.~Sutskever, and G.~E. Hinton.
\newblock Imagenet classification with deep convolutional neural networks.
\newblock In {\em Advances in neural information processing systems}, pages
  1097--1105, 2012.

\bibitem{Kuzborskij13a}
I.~Kuzborskij, F.~Orabona, and B.~Caputo.
\newblock From n to n+ 1: {Multiclass} transfer incremental learning.
\newblock In {\em Computer Vision and Pattern Recognition (CVPR), 2013 IEEE
  Conference on}, pages 3358--3365. IEEE, 2013.

\bibitem{levy2012minimal}
N.~Levy and L.~Wolf.
\newblock Minimal correlation classification.
\newblock In {\em Computer Vision--ECCV 2012}, pages 29--42. Springer, 2012.

\bibitem{nin}
M.~Lin, Q.~Chen, and S.~Yan.
\newblock Network in network.
\newblock In {\em International Conference on Learning Representations (ICLR)},
  2013.

\bibitem{baidu}
J.~Liu, Y.~Deng, T.~Bai, and C.~Huang.
\newblock Targeting ultimate accuracy: Face recognition via deep embedding.
\newblock {\em CoRR}, abs/1506.07310, 2015.

\bibitem{fcn}
J.~Long, E.~Shelhamer, and T.~Darrell.
\newblock Fully convolutional networks for semantic segmentation.
\newblock {\em CVPR (to appear)}, Nov. 2015.

\bibitem{Word2Vec}
T.~Mikolov, I.~Sutskever, K.~Chen, G.~S. Corrado, and J.~Dean.
\newblock Distributed representations of words and phrases and their
  compositionality.
\newblock In {\em Advances in Neural Information Processing Systems}, pages
  3111--3119, 2013.

\bibitem{Orabona09}
F.~Orabona, C.~Castellini, B.~Caputo, A.~E. Fiorilla, and G.~Sandini.
\newblock Model adaptation with least-squares {SVM} for adaptive hand
  prosthetics.
\newblock In {\em Robotics and Automation, 2009. ICRA'09. IEEE International
  Conference on}, pages 2897--2903. IEEE, 2009.

\bibitem{facenet}
F.~Schroff, D.~Kalenichenko, and J.~Philbin.
\newblock Facenet: A unified embedding for face recognition and clustering.
\newblock June 2015.

\bibitem{vgg}
K.~Simonyan and A.~Zisserman.
\newblock Very deep convolutional networks for large-scale image recognition.
\newblock In {\em International Conference on Learning Representations (ICLR)},
  2015.

\bibitem{deepid2}
Y.~Sun, Y.~Chen, X.~Wang, and X.~Tang.
\newblock Deep learning face representation by joint
  identification-verification.
\newblock In Z.~Ghahramani, M.~Welling, C.~Cortes, N.~Lawrence, and
  K.~Weinberger, editors, {\em Advances in Neural Information Processing
  Systems 27}, pages 1988--1996. Curran Associates, Inc., 2014.

\bibitem{deepid3}
Y.~Sun, D.~Liang, X.~Wang, and X.~Tang.
\newblock Deepid3: Face recognition with very deep neural networks.
\newblock {\em CoRR}, abs/1502.00873, 2015.

\bibitem{celebfaces}
Y.~Sun, X.~Wang, and X.~Tang.
\newblock Hybrid deep learning for face verification.
\newblock In {\em Computer Vision (ICCV), 2013 IEEE International Conference
  on}, pages 1489--1496, Dec 2013.

\bibitem{deepid1}
Y.~Sun, X.~Wang, and X.~Tang.
\newblock Deep learning face representation from predicting 10,000 classes.
\newblock In {\em The IEEE Conference on Computer Vision and Pattern
  Recognition (CVPR)}, June 2014.

\bibitem{deepid2plus}
Y.~Sun, X.~Wang, and X.~Tang.
\newblock Deeply learned face representations are sparse, selective, and
  robust.
\newblock In {\em The IEEE Conference on Computer Vision and Pattern
  Recognition (CVPR)}, June 2015.

\bibitem{GoogleLeNet}
C.~Szegedy, W.~Liu, Y.~Jia, P.~Sermanet, S.~Reed, D.~Anguelov, D.~Erhan,
  V.~Vanhoucke, and A.~Rabinovich.
\newblock Going deeper with convolutions.
\newblock In {\em CVPR 2015}, 2015.

\bibitem{multibox}
C.~Szegedy, S.~Reed, D.~Erhan, and D.~Anguelov.
\newblock Scalable, high-quality object detection.
\newblock {\em CoRR}, abs/1412.1441, 2014.

\bibitem{deepface}
Y.~Taigman, M.~Yang, M.~Ranzato, and L.~Wolf.
\newblock Deepface: Closing the gap to human-level performance in face
  verification.
\newblock In {\em Proceedings of the 2014 IEEE Conference on Computer Vision
  and Pattern Recognition}, CVPR '14, pages 1701--1708, Washington, DC, USA,
  2014. IEEE Computer Society.

\bibitem{Taigman_2015_CVPR}
Y.~Taigman, M.~Yang, M.~Ranzato, and L.~Wolf.
\newblock Web-scale training for face identification.
\newblock In {\em The IEEE Conference on Computer Vision and Pattern
  Recognition (CVPR)}, June 2015.

\bibitem{Tommasi10}
T.~Tommasi, F.~Orabona, and B.~Caputo.
\newblock Safety in numbers: Learning categories from few examples with multi
  model knowledge transfer.
\newblock In {\em Computer Vision and Pattern Recognition (CVPR), 2010 IEEE
  Conference on}, pages 3081--3088. IEEE, 2010.

\bibitem{eightymil}
A.~Torralba, R.~Fergus, and W.~Freeman.
\newblock 80 million tiny images: A large data set for nonparametric object and
  scene recognition.
\newblock {\em Pattern Analysis and Machine Intelligence, IEEE Transactions
  on}, 30(11):1958--1970, Nov 2008.

\bibitem{tsne}
L.~van~der Maaten and G.~E. Hinton.
\newblock Visualizing high-dimensional data using t-sne.
\newblock {\em Journal of Machine Learning Research}, 9:2579--2605, 2008.

\bibitem{vedaldi15matconvnet}
A.~Vedaldi and K.~Lenc.
\newblock Matconvnet -- convolutional neural networks for matlab.

\bibitem{jain}
D.~Wang, C.~Otto, and A.~K. Jain.
\newblock Face search at scale: 80 million gallery.
\newblock {\em CoRR}, abs/1507.07242, 2015.

\bibitem{Yang07}
J.~Yang, R.~Yan, and A.~G. Hauptmann.
\newblock Cross-domain video concept detection using adaptive svms.
\newblock In {\em Proceedings of the 15th international conference on
  Multimedia}, pages 188--197. ACM, 2007.

\bibitem{cardb}
L.~Yang, P.~Luo, C.~Change~Loy, and X.~Tang.
\newblock A large-scale car dataset for fine-grained categorization and
  verification.
\newblock June 2015.

\bibitem{casia}
D.~Yi, Z.~Lei, S.~Liao, and S.~Z. Li.
\newblock Learning face representation from scratch.
\newblock {\em CoRR}, abs/1411.7923, 2014.

\bibitem{c40}
E.~Zhou, Z.~Cao, and Q.~Yin.
\newblock Naive-deep face recognition: Touching the limit of {LFW} benchmark or
  not?
\newblock {\em CoRR}, abs/1501.04690, 2015.

\bibitem{c45}
Z.~Zhu, P.~Luo, X.~Wang, and X.~Tang.
\newblock Recover canonical-view faces in the wild with deep neural networks.
\newblock {\em CoRR}, abs/1404.3543, 2014.

\end{thebibliography}
}

\end{document}